\documentclass[11pt]{article} 
\usepackage{rldmsubmit,palatino}
\usepackage{graphicx}
\usepackage{algorithm}
\usepackage{algpseudocode}
\usepackage{mathtools}
\usepackage{mhsetup}
\usepackage{amsfonts}
\usepackage{amsmath}
\usepackage{amsthm}
\usepackage{mathrsfs}
\usepackage{sansmath}
\usepackage{thmtools}
\usepackage{tikz} 
\usepackage{multicol}
\usepackage{subfig}
\usepackage{booktabs}

\newtheoremstyle{thm}
  {8pt}   			
  {0}   		
  {}  		
  {}       			
  {\bfseries} 
  {.} 
  {.5em} 
  {}          			

\makeatletter
\renewenvironment{proof}[1][\proofname]{\par
  \vspace{0pt}
  \pushQED{\qed}%
  \normalfont
  \topsep0pt \partopsep0pt 
  \trivlist
  \item[\hskip\labelsep
        \itshape
    #1\@addpunct{.}]\ignorespaces
}{%
  \popQED\endtrivlist\@endpefalse
  \addvspace{2pt} 
}

\makeatother\newtheorem{definition}{Definition} 

\theoremstyle{thm}

\newtheorem{lemma}[definition]{Lemma}
\newtheorem{theorem}[definition]{Theorem}

\DeclareMathOperator*{\argmax}{arg\,max}

\title{Unsupervised Basis Function Adaptation \\for Reinforcement Learning}

\author{
Edward W.~Barker \\
School of Mathematics and Statistics \\
University of Melbourne\\
Melbourne, Australia \\
\texttt{ebarker@student.unimelb.edu.au} \\
\And
Charl J.~Ras \\
School of Mathematics and Statistics \\
University of Melbourne\\
Melbourne, Australia \\
\texttt{cjras@unimelb.edu.au} \\
}

\begin{document}

\maketitle

\begin{abstract}
When using reinforcement learning (RL) algorithms to evaluate a policy it is common, given a large state space, to introduce some form of approximation architecture for the value function (VF).  The exact form of this architecture can have a significant effect on the accuracy of the VF estimate, however, and determining a suitable approximation architecture can often be a highly complex task.  Consequently there is a large amount of interest in the potential for allowing RL algorithms to adaptively generate approximation architectures.  

We investigate a method of adapting approximation architectures which uses feedback regarding the frequency with which an agent has visited certain states to guide which areas of the state space to approximate with greater detail.  This method is ``unsupervised'' in the sense that it makes no direct reference to reward or the VF estimate.  We introduce an algorithm based upon this idea which adapts a state aggregation approximation architecture on-line.  

A common method of scoring a VF estimate is to weight the squared Bellman error of each state-action by the probability of that state-action occurring.  Adopting this scoring method, and assuming $S$ states, we demonstrate theoretically that --- provided (1) the number of cells $X$ in the state aggregation architecture is of order $\sqrt{S}\log_2{S}\ln{S}$ or greater, (2) the policy and transition function are close to deterministic, and (3) the prior for the transition function is uniformly distributed --- our algorithm, used in conjunction with a suitable RL algorithm, can guarantee a score which is arbitrarily close to zero as $S$ becomes large.  It is able to do this despite having only $O(X \log_2S)$ space complexity and negligible time complexity.  The results take advantage of certain properties of the stationary distributions of Markov chains. 
\end{abstract}

\keywords{
reinforcement learning, unsupervised learning, basis function adaptation, state aggregation
}

\startmain 

\section{Introduction}

When using traditional reinforcement learning (RL) algorithms (such as $Q$-learning or SARSA) to evaluate policies in environments with large state or action spaces, it is common to introduce some form of architecture with which to approximate the value function (VF), for example a parametrised set of functions.  This approximation architecture allows algorithms to deal with problems which would otherwise be computationally intractable.  One issue when introducing VF approximation, however, is that the accuracy of the algorithm's VF estimate is highly dependent upon the exact form of the architecture chosen.  

Accordingly, a number of authors have explored the possibility of allowing the approximation architecture to be \emph{learned} by the agent, rather than pre-set manually by the designer (see \cite{Bucsoni2009} for an overview).  It is common to assume that the approximation architecture being adapted is linear (so that the value function is represented as a weighted sum of basis functions) in which case such methods are known as \emph{basis function adaptation}.  When designing a method to adapt approximation architectures, we assume that we have an underlying RL algorithm which, for a given linear architecture, will generate a VF estimate.  If we assume that our basis function adaptation occurs on-line, then the RL algorithm will be constantly updating its VF estimate as the basis functions are updated (the latter typically on a slower time-scale).  

A simple and perhaps, as yet, under-explored method of basis function adaptation involves using an estimate of the frequency with which an agent has visited certain states to determine which states to more accurately represent.  Such methods are \emph{unsupervised} in the sense that no direct reference to the reward or to any estimate of the value function is made.  The concept of using visit frequencies in an unsupervised manner is not completely new \cite{Menache2005}\cite{Bernstein2009} however it remains relatively unexplored when compared to methods based on direct estimation of VF error \cite{Munos2002}\cite{BertYu2009}\cite{DiCastro2010}\cite{Mahadevan2013}. 

However it is possible that such methods can offer some unique advantages.  In particular: (i) estimates of visit frequencies are very cheap to calculate, (ii) accurate estimates of visit frequencies can be generated with a relatively small number of samples, and, perhaps most importantly, (iii) in many cases visit frequencies contain a lot of the most important information regarding where accuracy is required in the VF estimate.  Our aim here is to further explore and quantify, where possible, these advantages.  In the next section we outline the details of an algorithm (PASA, short for ``Probabilistic Adaptive State Aggregation'') which performs unsupervised basis function adaptation based on state aggregation.  This algorithm will form the basis upon which we develop theoretical results in Section \ref{theoretic}.  

\section{The PASA algorithm}
\label{algorithm}

\subsection{Formal setting}

The agent interacts with an environment over a sequence of iterations $t \in \mathbb{N}$.  For each $t$ it will be in a particular state $s_i$ ($1 \leq i \leq S$) and will take a particular action $a_j$ ($1 \leq j \leq A$) according to a policy $\pi$ (we denote as $\pi(a_j|s_i)$ the probability the agent takes action $a_j$ in state $s_i$).  Transition and reward functions are denoted as $P$ and $R$ respectively (and are unknown, however we assume we are given a prior distribution for both).  Hence we use $P(s_{i'}|s_i,a_j)$ to denote the probability the agent will transition to state $s_{i'}$ given it takes action $a_j$ in state $s_i$.  We assume the reward function (which maps each state-action pair to a real number) is bounded: $|R(s_i,a_j)| < \infty$ for all $(i,j)$.  We are considering the problem of policy evaluation, so we will always assume that the agent's policy $\pi$ is fixed.  A \emph{state aggregation} approximation architecture we will define as a mapping $F$ from each state $s_i$ to a \emph{cell} $x_k$ ($1 \leq k \leq X$), where typically $X \ll S$ \cite{Whiteson2007}.  Given a state aggregation approximation architecture, an RL algorithm will maintain an estimate $\hat{Q}_{\theta}$ of the true value function $Q^{\pi}$ \cite{Bucsoni2009}, where $\theta$ specifies a weight associated with each cell-action pair.  Provided for any particular mapping $F$ the value $\theta$ converges, then each mapping $F$ has a corresponding VF estimate for the policy $\pi$.  

Many different methods can be used to score a VF estimate (i.e. measure its accuracy).  A common score used is the squared \emph{Bellman error} \cite{Bucsoni2009} for each state-action, weighted by the probability of visiting each state.  This is called the \emph{mean squared error} (MSE).  When the true VF $Q^{\pi}$ is unknown we can use $T$, the \emph{Bellman operator}, to obtain an approximation of the MSE.\footnote{See, for example, score functions proposed by Menache et al \cite{Menache2005} and Di Castro et al \cite{DiCastro2010}.}  This approximation we denote as $L$.  Hence:
\begin{multline}
\label{MSE}
\text{MSE} \coloneqq \sum_{i=1}^S \psi_i \sum_{j=1}^A \left(Q^{\pi}(s_{i},a_{j}) - \hat{Q}_{\theta}(s_{i},a_{j}) \right)^2 \approx L \coloneqq \sum_{i=1}^S \psi_i \sum_{j=1}^A \left( T\hat{Q}_{\theta}(s_{i},a_{j}) - \hat{Q}_{\theta}(s_{i},a_{j}) \right)^2 \\
= \sum_{i=1}^S \psi_i \sum_{j=1}^A \Bigg( R(s_i,a_j) + \gamma \sum_{i'=1}^S\sum_{j'=1}^A P(s_{i'}|s_i,a_j)\pi(a_{j'}|s_{i'}) \hat{Q}_{\theta}(s_{i'},a_{j'}) - \hat{Q}_{\theta}(s_i,a_j) \Bigg)^2 
\end{multline}

This is where $\gamma \in [0,1)$ is a discount factor and where $\psi$ is a vector of the probability of each state given the stationary distribution associated with $\pi$ (given some fixed policy $\pi$, the transition matrix, obtained from $\pi$ and $P$, has a corresponding stationary distribution).  We assume that our task in developing an adaptive architecture is to design an algorithm which will adapt $F$ so that $L$ is minimised.  The PASA algorithm, which we now outline, adapts the mapping $F$.  

\subsection{Details of PASA algorithm}

PASA will store a vector $\rho$ of integers of dimension $X - B$, where $B < X$.  Suppose we start with a partition of the state space into $B$ cells, indexed from $1$ to $B$, each of which is approximately the same size.  Using $\rho$ we can now define a new partition by splitting (as evenly as possible) the $\rho_1$th cell in this partition.  We leave one half of the $\rho_1$th cell with the index $\rho_1$ and give the other half the index $B+1$ (all other indices stay the same).  Taking this new partition (consisting of $B+1$ cells) we can create a further partition by splitting the $\rho_2$th cell.  Continuing in this fashion we will end up with a partition containing $X$ cells (which gives us the mapping $F$).  We need some additional mechanisms to allow us to update $\rho$.  Denote as $\mathcal{X}_{i,j}$ the set of states in the $i$th cell of the $j$th partition (so $0 \leq j \leq X-B$ and $1 \leq i \leq B+j$).  The algorithm will store a vector $\bar{u}$ of real values of dimension $X$.  This will record the approximate frequency with which certain cells have been visited by the agent.  We define a new vector $\bar{x}$ of dimension $X$:
\begin{equation}
\bar{x}_i(t) = 
\begin{cases} 
	I_{\{s(t) \in \mathcal{X}_{i,0}\}} & \text{if } 1 \leq i \leq B \\
	I_{\{s(t) \in \mathcal{X}_{i,i-B}\}} & \text{if } B < i \leq X 
\end{cases}
\end{equation}
where $I$ is the indicator function for a logical statement (such that $I_{A} = 1$ if $A$ is true).  The resulting mapping from each state to a vector $\bar{x}$ we denote as $\bar{F}$.  We then update $\bar{u}$ in each iteration as follows (i.e. using a simple stochastic approximation algorithm):
\begin{equation}
\bar{u}_i(t+1) = \bar{u}_i(t) + \eta \left(\bar{x}_i(t) - \bar{u}_i(t) \right)  
\end{equation}

This is where $\eta \in (0,1]$ is a constant step size parameter.  To update $\rho$, at certain intervals $\nu \in \mathbb{N}$ the PASA algorithm performs a sequence of $X - B$ operations.  A temporary copy of $\bar{u}$ is made, which we call $u$.  We also store an $X$ dimensional boolean vector $\Sigma$ and set each entry to zero at the start of the sequence (this keeps track of whether a particular cell has only one state, as we don't want singleton cells to be split).  At each stage $k \geq 1$ of the sequence we update $\rho$, $u$ and $\Sigma$, in order, as follows (for $\rho$, if multiple indices satisfy the $\argmax$ function, we take the lowest index):
\begin{equation}
\begin{split}
&\rho_k = 
\begin{cases} 
	\argmax_i\{u_i:i \leq B + k - 1, \Sigma_i = 0\} & \text{if } (1-\Sigma_{\rho_k})u_{\rho_k} < \max\{u_i:i \leq B + k - 1, \Sigma_i = 0\} - \vartheta \\
	\rho_k & \text{otherwise}   
\end{cases} \\
&u_{\rho_k} \gets u_{\rho_k} - u_{B+k} \quad \quad \Sigma_i = I_{\{|\mathcal{X}_{i,k}| \leq 1\}} \text { for } 1 \leq i \leq B + k
\end{split}
\end{equation}
where $\vartheta > 0$ is a constant designed to ensure that a (typically small) threshold must be exceeded before $\rho$ is adjusted.  The idea behind each step in the sequence is that the non-singleton cell $\mathcal{X}_{i,k}$ with the highest value $u_i$ (an estimate of visit frequency which is recalculated at each step) will be split.  Details of these steps, as well as the overall PASA process, are outlined in Algorithm \ref{PASA}.  Note that the algorithm calls a procedure to \textproc{Split} cells.  This procedure simply updates $\bar{F}$ and $\Sigma$ given the latest value of $\rho$.  It also calls a \textproc{Convert} procedure, which converts the mapping $\bar{F}$ to a mapping $F$.

\begin{algorithm}
\caption{The PASA algorithm.  Called at each iteration $t$.  Assumes $\bar{u}$, $\bar{F}$, $F$ and $\rho$ are stored.  Return is void.}
\label{PASA}
\vspace{-14pt}
\begin{multicols}{2}
\begin{algorithmic}[1]
	\Function{PASA}{$t$,$s$,$\eta$,$\vartheta$,$\nu$}
		\State $\bar{x} \gets \bar{F}(s)$
		\State $\bar{u} \gets \bar{u} + \eta(\bar{x} - \bar{u})$
		\If {$t \bmod \nu = 0$}
			\State $u \gets \bar{u}$
			\For {$k \in \{1,X\}$}
				\State $\Sigma_k \gets 0$
			\EndFor
			\For {$k \in \{1,X-B\}$}
				\State $u_{\text{max}} \gets \max\{u_i: \Sigma_i = 0\}$
				\State $i_{\text{max}} \gets \min\{i:u_i = u_{\text{max}}, \Sigma_i = 0\}$
				\If {$u_{i_{\text{max}}} - \vartheta > (1-\Sigma_{\rho_k})u_{\rho_k}$}
					\State $\rho_k \gets i_{\text{max}}$
				\EndIf
				\State $u_{\rho_k} \gets u_{\rho_k} - u_{B+k}$
				\State $(\bar{F},\Sigma) \gets \Call{Split}{k,\rho,\bar{F}}$
			\EndFor
			\State $F \gets \Call{Convert}{\bar{F}}$
		\EndIf
	\EndFunction
\end{algorithmic}
\end{multicols}
\vspace{-10pt}
\end{algorithm}

\subsection{Some basic properties of PASA}

PASA requires only a modest increase in computational resources compared to fixed state aggregation.  In relation to time complexity, $\bar{u}$ can be updated in parallel with the RL algorithm's update of $\theta$ (and the update of $\bar{u}$ would not be expected to have any greater time complexity than the update to $\theta$ if using a standard RL algorithm such as SARSA), whilst $\rho$ can be updated at large intervals $\nu$ (and this update can also be run in parallel).  Applying the mapping $F$ to a state has a very low order of time complexity --- $O(\log_2S)$ for an RL algorithm using PASA compared to $O(\log_2X)$ for $X$ equally sized cells.  Hence, PASA involves no material increase in time complexity.

PASA does involve additional space complexity with respect to storing the vector $\bar{u}$:  we must store $X$ real values.  If we also store $F$ and $\bar{F}$ (as well as $u$ temporarily) the overall space complexity becomes $O(X\log_2{S})$.  The RL component has space complexity $O(XA)$ (reflecting the $X \times A$ cell-action pairs), so that the introduction of PASA as a pre-processing algorithm will not impact the overall space complexity at all if $A > \log_2S$.  (Note also that the space complexity of PASA is independent of $A$.)  Regarding sampling efficiency, since methods based on explicitly estimating the Bellman error (or MSE) require a VF estimate (generated by the RL algorithm), as well as information about reward for all actions in the action space, we can expect PASA (and other unsupervised methods) to require comparatively less sampling to generate the estimates it requires to update the approximation architecture.  

We can reassure ourselves (somewhat informally) that PASA will converge (for fixed $\pi$) in the the following sense.  We can set $\eta$ small enough so that the sum of the elements $|\bar{u}_i|$ for $1 \leq i \leq X$ remains within some interval of size $\vartheta/2$ over some arbitrarily large number of iterations with arbitrarily high probability (after allowing a sufficient number of iterations) for all possible $\rho$.  Then $\rho_1$ will eventually remain the same with arbitrarily high probability (i.e. is ``fixed'').  Suppose $\rho_i$ remains fixed for $1 \leq i \leq k$.  Then, since each element $u_j$ (for $1 \leq j \leq B + k$) of $u$ calculated at the $k$th step of the sequence described in Algorithm \ref{PASA} will remain within an interval of size $\vartheta/2$ with arbitrarily high probability, the value $\rho_{k+1}$ will also remain fixed.  Hence, by induction, there exists $\eta$ such that $\rho$ will eventually remain fixed.  

\section{Result regarding Bellman error for policy evaluation}
\label{theoretic}

We now set out our main result.  The key idea is that, in many important circumstances, which are reflective of real world problems, when following a fixed policy (even when this is generated randomly) an agent will have a tendency to spend nearly all of its time in only a small subset of the state space.  We can use this property to our advantage.  It means that by focussing on this small area (which is what PASA does) we can eliminate most of the terms which significantly contribute to $L$.  The trick will be to quantify this tendency.

We must make the following assumptions: (1) $P$ is ``close to'' deterministic (i.e. $P$ can expressed by a deterministic transition function $P_1$, which at each $t$ is applied with probability $1 - \delta$, and an arbitrary transition function $P_2$ which is applied with probability $\delta$, where $\delta$ is small;  what constitutes ``small'' will be made clearer below), (2) $P$ has a uniform \emph{prior} distribution, in the sense that, according to our prior distribution for $P$, for each $(s_i,a_j)$, $\mathrm{Pr}(P(s_{i'}|s_i,a_j))$ is independently distributed and $\mathrm{Pr}(P(s_{i'}|s_i,a_j)=p) = \mathrm{Pr}(P(s_{i''}|s_i,a_j)=p)$ for all $p \in [0,1]$, $s_{i'}$ and $s_{i''}$, and (3) $\pi$ is also ``close to'' deterministic (i.e. the probability of \emph{not} taking the most probable action is no greater than $\delta$ for each state).  

We can make the following observation.  If $\pi$ and $P$ are deterministic, and we pick a starting state $s_1$, then the agent will create a path through the state space and will eventually revisit a previously visited state, and will then enter a cycle.  Call the set of states in this cycle $\mathcal{C}_1$ and denote as $C_1$ the number of states in the cycle.  If we now place the agent in a state $s_2$ (arbitrarily chosen) it will either create a new cycle or it will terminate on the path or cycle created from $s_1$.  Call $\mathcal{C}_2$ the states in the second cycle (and $C_2$ the number of states in the cycle, noting that $C_2 = 0$ is possible).  If we continue in this manner we will have $S$ sets $\{\mathcal{C}_1,\mathcal{C}_2,\ldots,\mathcal{C}_S\}$.  Call $\mathcal{C}$ the union of these sets and denote as $C$ the number of states in $\mathcal{C}$.  We denote as $T_i$ the event that the $i$th path created in such a manner terminates on itself, and note that, if this does not occur, then $C_i = 0$.  

If (2) holds then $\mathrm{E}(C_1) = \sqrt{\pi S/8} + O(1)$ and $\mathrm{Var}(C_1) = (32-8\pi)S/24 + O(\sqrt{S})$.\footnote{This follows from the solution to the ``birthday problem'' (the solution gives the mean length of the path, and since each cycle length has equal probability when conditioned on this path length, we can divide the mean by $2$).  The expectation is over the prior distribution for $P$.  For a description of the problem and a formal proof see, for example, page 114 of Flajolet and Sedgewick \cite{Flajolet2009}.  The variance can be derived from first principles using similar techniques to those used for the mean in the birthday problem.}  Supposing that $\pi$ and $P$ are no longer deterministic then, if (1) and (3) hold, we can set $\delta$ sufficiently low so that the agent will spend an arbitrarily large proportion of its time in $\mathcal{C}$.  If (2) holds we also have the following: 

\begin{lemma}
\label{genmoments}
$\mathrm{E}(C) < \mathrm{E}(C_1)(\ln{S}+1)$ and $\mathrm{Var}(C) \leq O(S\ln S)$.
\end{lemma}

\begin{proof}
We will have:
\begin{equation}
\begin{split}
\mathrm{E}(C) &= \sum_{i=1}^{S}\mathrm{E}(C_i) = \sum_{i=1}^{S}\mathrm{Pr}(T_i)\sum_{j=1}^{S}j\mathrm{Pr}(C_i = j|T_i) \leq \sum_{i=1}^{S}\frac{1}{i}\sum_{j=1}^{S}j\mathrm{Pr}(C_1 = j) < \mathrm{E}(C_1)(\ln S + 1)
\end{split}
\end{equation}

And for the variance:
\begin{equation}
\begin{split}
\mathrm{Var}(C) &= \sum_{i=1}^{S}\mathrm{Var}(C_i) + 2\sum_{i=2}^{S}\sum_{j=1}^{i-1} \mathrm{Cov}(C_iC_j) \leq \sum_{i=1}^{S}\mathrm{Var}(C_i) \leq \sum_{i=1}^{S}\mathrm{E}(C_i^2) = \sum_{i=1}^{S}\mathrm{Pr}(T_i)\sum_{j=1}^{S}j^2\mathrm{Pr}(C_i = j|T_i) \\
&\leq \sum_{i=1}^{S}\frac{1}{i}\sum_{j=1}^{S}j^2\mathrm{Pr}(C_1 = j) < \mathrm{E}(C_1^2)(\ln S + 1) = \left(\mathrm{Var}(C_1) + \mathrm{E}(C_1)^2 \right)(\ln S + 1)
\end{split}
\end{equation}
where we have used the fact that the covariance term must be negative for any pair of lengths $C_i$ and $C_j$, since if $C_i$ is greater than its mean the expected length of $C_j$ must decrease, and vice versa.
\end{proof}

\begin{theorem}
\label{error}
For all $\epsilon_1 > 0$ and $\epsilon_2 > 0$, there is sufficiently large $S$ and sufficiently small $\delta$ such that PASA in conjunction with a suitable RL algorithm will --- provided $X \geq K\sqrt{S}\ln{S}\log_2{S}$ for some $K > \sqrt{\pi/8}$ --- generate, with probability no less than $1 - \epsilon_1$, a VF estimate with $L \leq \epsilon_2$.
\end{theorem}

\begin{proof}
Using Chebyshev's inequality, and Lemma \ref{genmoments}, we can choose $S$ sufficiently high so that $C > K\sqrt{S}\ln{S}$ with probability no greater than $\epsilon_1$.  Since $R$ is bounded and $\gamma < 1$ then for any $\epsilon_2 > 0$, $F$ and $S$ we can also choose $\delta$ so that $L$ summed only over states not in $\mathcal{C}$ is no greater than $\epsilon_2$.  We choose $\delta$ so that this is satisfied, but also so that $\psi_i > \sum_{i':s_{i'} \notin \mathcal{C}}\psi_{i'}$ for all elements of $\{\psi_i:s_i \in \mathcal{C}\}$.\footnote{Each such $\psi_i$ will be bounded from below for all $\delta > 0$.  This can be verified by more closely examining the geometric distributions which govern the ``jumping'' between distinct cycles.}  Now provided that $C\log_2S \leq X$ then each state in $\mathcal{C}$ will eventually be in its own cell.  The RL algorithm will have no error for each such state so therefore $L$ will be no greater than $\epsilon_2$.  
\end{proof}

The bound on $X$ provided represents a significant reduction in complexity when $S$ starts to take on a size comparable to many real world problems (and could make the difference between a problem being tractable and intractable).  It also seems likely that the bound on $X$ in Theorem \ref{error} can be improved upon, as the one provided is not necessarily as tight as possible.  Conditions (1) and (3) are commonly encountered in practice, in particular (3) which can be taken to reflect a ``greedy'' policy.  Condition (2) can be interpreted as the transition function being ``completely unknown'' (it seems possible that similar results may hold under other, more general, assumptions regarding the prior distribution).  Note finally that the result can be extended to exact MSE if MSE is redefined so that it is also weighted by $\pi$. 

\section{Discussion}

The key message from our discussion is that there are commonly encountered circumstances where unsupervised methods can be very effective in creating an approximation architecture.  However, given their simplicity, they can at the same time avoid the cost (both in terms of computational complexity, and sampling required) associated with more complex adaptation methods.  In the setting of policy \emph{improvement} these advantages have the potential to be particularly important, especially when dealing with large state spaces.  Some initial experimentation suggests that the PASA algorithm can have a significant impact on RL algorithm performance in both policy evaluation and policy improvement settings.

The nature of the VF estimate generated by PASA and its associated RL algorithm is that the VF will be well estimated for states which are visited frequently under the existing policy.  This does come at a cost, however, as estimates of the value of deviating from the current policy will be made less accurate.  Thus, even though $L$ or MSE may be low, it does not immediately follow that an algorithm can use this to optimise its policy via standard policy iteration (since the consequences of deviating from the current policy are less clearly represented).  Ultimately, however, the theoretical implications of the improved VF estimate in the context of policy iteration are complex, and would need to be the subject of further research.

\end{document}